\documentclass[accepted]{uai2024}
\usepackage{fancyhdr}
\fancypagestyle{plain}{\fancyhf{}\cfoot{\thepage}}

\usepackage[british]{babel}

\usepackage[T1]{fontenc}

\usepackage{natbib} 
\usepackage{mathtools} 
\usepackage{booktabs} 
\usepackage{tikz} 

\usepackage{graphicx}
\usepackage{balance} 
\usepackage{tcolorbox}
\usepackage{alltt}
\usepackage{array}
\usepackage{amsmath}
\usepackage[whole]{bxcjkjatype}
\usepackage{comment}
\usepackage{bm}
\usepackage{amssymb}
\usepackage{cases}

\usepackage{tikz}
\usepackage{pgfplots}
\usepackage{amsmath,amssymb,amsfonts,amsthm}

\newtheorem{example}{Example}
\newtheorem{theorem}{Theorem}
\newtheorem{definition}{Definition}
\newtheorem{lemma}{Lemma}
\newtheorem{proposition}{Proposition}
\newtheorem{corollary}{Corollary}
\newcommand{\cent}{\mathrel{\scalebox{1}[1.5]{$\shortmid$}\mkern-3.1mu\raisebox{0.1ex}{$=$}}}
\newcommand{\ent}{\mathrel{\scalebox{1}[1.5]{$\shortmid$}\mkern-3.1mu\raisebox{0.1ex}{$\equiv$}}}

\usepackage{stmaryrd}
\usepackage{bm}
\usepackage[whole]{bxcjkjatype}
\usepackage{cancel}

\newcommand{\m}[1]{[\![#1]\!]}
\newcommand{\ms}[1]{[\![#1]\!]}
\newcommand{\pms}[1]{[\![\![#1]\!]\!]}
\newcommand{\ams}[1]{(\!(#1)\!)}
\newcommand{\pams}[1]{(\!(\!(#1)\!)\!)}

\title{Inference of Abstraction for a Unified Account of Symbolic Reasoning from Data}
\author[ ]{Hiroyuki Kido}
\affil[ ]{%
    Cardiff University\\
    Park Place, Cardiff, CF10 3AT, UK
}

\begin{document}
\maketitle
\begin{abstract}
Inspired by empirical work in neuroscience for Bayesian approaches to brain function, we give a unified probabilistic account of various types of symbolic reasoning from data. We characterise them in terms of formal logic using the classical consequence relation, an empirical consequence relation, maximal consistent sets, maximal possible sets and maximum likelihood estimation. The theory gives new insights into reasoning towards human-like machine intelligence.
\end{abstract}
\section{Introduction}\label{sec:introduction}
There is growing evidence that the brain is a generative model of environments. The image shown in Figure \ref{fig:illusions} would cause the perception that a white triangle overlays the other objects. A well-accepted explanation of the illusion is that our brains are trained to unconsciously use past experience to see what is likely to happen. Much empirical work argues that Bayesian, or probabilistic generative, models give a clear explanation of how the brain reconciles top-down prediction signals and bottom-up sensory signals, e.g., \citep{Helmholtz:25,Knill:96,Gregory:97,Knill:04,lee:03,Hohwy:14,friston:10,Rao:99,Rao:05,Caucheteux:23,Itti:09,Adams:16,Pellicano:12,Tenenbaum:06}.
\par
An interesting question emerging from this idea is how logical consequence relations, relevant to human higher-order thinking, can be given a Bayesian account. The question is important for the following reasons. First, such an account should result in a simple computational principle that allows logical agents to reason over symbolic knowledge fully from data in an uncertain environment. Second, such a principle is expected to tackle fundamental assumptions of the existing computational models such as statistical relational learning (SRL) \citep{Getoor:07}, Bayesian networks \citep{pearl:88}, naive Bayes, probabilistic logic programming (PLP) \citep{sato:95}, Markov logic networks (MLN) \citep{richardson:06}, probabilistic logic \citep{nilsson:86}, probabilistic relational models (PRM) \citep{friedman:99} and conditional probabilistic logic \citep{Rodder:00}. For example, they have the implicit assumption that the method used to extract symbolic knowledge from data cannot be applied to the method used to perform logical reasoning over the symbolic knowledge, and vice versa.
\par
In this paper, we simply model how data cause symbolic knowledge in terms of its satisfiability in formal logic. The underlying idea is to see reasoning as a process of deriving symbolic knowledge from data via abstraction, i.e., selective ignorance. We show that various types of well-grounded symbolic reasoning emerge from direct interaction between data and symbols, not between symbols. We theoretically characterise them in terms of formal logic.
\par
This paper contributes to new insights into reasoning towards human-like machine intelligence. Symbolic reasoning is essentially a reference to data in our theory. It thus brings up an idea of inference grounding rather than or beyond symbol grounding. Symbolic reasoning can also be seen as interaction between an interpretation in formal logic and its inversion. The inversion, we call inverse interpretation, differentiates our work from the mainstream referred to as inverse entailment \citep{Muggleton:95}, inverse resolution \citep{Muggleton:88,Nienhuys:97} and inverse deduction \citep{Russell:20}, which mainly study dependency between pieces of symbolic knowledge. Our analysis causes reasoning from an impossible source of information, which may be coined as parapossible reasoning, since reasoning from an inconsistent source of information is often referred to as paraconsistent reasoning \citep{Priest:02,Carnielli:07}.
\par
In Section 2, we define a generative reasoning model for inference of abstraction. Section 3 gives full logical characterisations of the theory. Section 4 summarises the results.
\begin{figure}[t]
\begin{center}
 \includegraphics[scale=0.25]{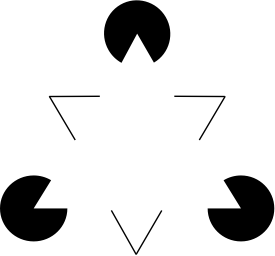}
    \caption{The Kanizsa triangle. Adapted from Kanizsa \citep{Kanizsa:55}.}
  \label{fig:illusions}
  \end{center}
\end{figure}

\section{Inference of Abstraction}
Let $\{d_{1},d_{2},...,d_{K}\}$ be a multiset of $K$ data. $D$ denotes a random variable of data whose values are all the elements of $\{d_{1},d_{2},...,d_{K}\}$. For all data $d_{k} (1\leq k\leq K)$, we define the probability of $d_{k}$, denoted by $p(D=d_{k})$, as follows. 
\begin{eqnarray*}
\textstyle{p(D=d_{k})=\frac{1}{K}}
\end{eqnarray*}
\par
Let $L$ represent a propositional language for simplicity, and $\{m_{1},m_{2},...,m_{N}\}$ be the set of models of $L$. A model is an assignment of truth values to all the atomic formulas in $L$. Intuitively, each model represents a different state of the world. We assume that each data $d_{k}$ supports a single model. We thus use a function $m$, $\{d_{1},d_{2},...,d_{K}\}\to\{m_{1},m_{2},...,m_{N}\}$, to map each data to the model supported by the data. $M$ denotes a random variable of models whose realisations are all the elements of $\{m_{1},m_{2},...,m_{N}\}$. For all models $m_{n} (1\leq n\leq N)$, we define the probability of $m_{n}$ given $d_{k}$, denoted by $p(M=m_{n}|D=d_{k})$, as follows.
{\small
\begin{eqnarray*}
&&p(M=m_{n}|D=d_{k})=
\begin{cases}
1 & \text{if } m_{n}=m(d_{k})\\
0 & \text{otherwise}
\end{cases}
\end{eqnarray*}
}
\par
The truth value of a propositional formula and first-order closed formula in classical logic is uniquely determined in a state of the world specified by a model of a language. Let $\alpha$ be a formula in $L$. We assume that $\alpha$ is a random variable whose realisations are 0 and 1 meaning false and true respectively. We use symbol $\ms{\alpha}$ to refer to the models of $\alpha$. Namely, $\ms{\alpha=1}$ and $\ms{\alpha=0}$ represent the set of models in which $\alpha$ is true and false, respectively. Let $\mu\in[0,1]$ be a variable, not a random variable. For all formulas $\alpha\in L$, we define the probability of each truth value of $\alpha$ given $m_{n}$, denoted by $p(\alpha|M=m_{n})$, as follows.
{\small
\begin{eqnarray*}
&&p(\alpha=1|M=m_{n})=
\begin{cases}
\mu & \text{if } m_{n}\in\llbracket\alpha=1\rrbracket\\
1-\mu & \text{otherwise }
\end{cases}
\\
&&p(\alpha=0|M=m_{n})=
\begin{cases}
\mu & \text{if } m_{n}\in\llbracket\alpha=0\rrbracket\\
1-\mu & \text{otherwise }
\end{cases}
\end{eqnarray*}
}
The above expressions can be simply written as a Bernoulli distribution with parameter $\mu\in[0,1]$, i.e.,
{\small
\begin{eqnarray*}
p(\alpha|M=m_{n})=\mu^{\ms{\alpha}_{m_{n}}}(1-\mu)^{1-\ms{\alpha}_{m_{n}}}.
\end{eqnarray*}
}
Here, the variable $\mu\in[0,1]$ plays an important role to relate various types of symbolic reasoning. We will see that $\mu=1$ relates to the classical consequence relation and its generalisation, and $\mu$ approaching 1, denoted by $\mu\to 1$, relates to its generalisation for reasoning from inconsistent sources of information and its further generalisation. We discuss them in the next section.
\par
In classical logic, given a model, the truth value of a formula does not change the truth value of another formula. Thus, in probability theory, the truth value of a formula $\alpha_{1}$ is conditionally independent of the truth value of another formula $\alpha_{2}$ given a model $M$, i.e., $p(\alpha_{1}|\alpha_{2},M,D)=p(\alpha_{1}|M,D)$ or equivalently $p(\alpha_{1},\alpha_{2}|M,D)=p(\alpha_{1}|M,D)p(\alpha_{2}|M,D)$. We therefore have
{\small
\begin{align}
\textstyle{p(L|M,D)=\prod_{\alpha\in L}p(\alpha|M,D).}\label{eq:1}
\end{align}
}
Moreover, in classical logic, the truth value of a formula depends on models but not data. Thus, in probability theory, the truth value of a formula $\alpha$ is conditionally independent of data $D$ given a model $M$, i.e., $p(\alpha|M,D)=p(\alpha|M)$. We thus have
{\small
\begin{align}
\textstyle{\prod_{\alpha\in L}p(\alpha|M,D)=\prod_{\alpha\in L}p(\alpha|M).}\label{eq:2}
\end{align}
}
Therefore, the full joint distribution, $p(L, M, D)$, can be written as follows.
{\small
\begin{align}
p(L, M, D)&=p(L|M, D)p(M|D)p(D)\nonumber\\
&\textstyle{=\prod_{\alpha\in L}p(\alpha|M)p(M|D)p(D)}\label{eq:3}
\end{align}
}
Here, the product rule (or chain rule) of probability theory is applied in the first equation, and Equations (\ref{eq:1}) and (\ref{eq:2}) in the second equation. As will be seen later, the joint distribution $p(L, M, D)$ is a probabilistic model of symbolic reasoning from data. We call the joint distribution a generative reasoning model for short. We often represent $p(L, M, D)$ as $p(L, M, D;\mu)$ if our discussion is relevant to $\mu$. We use symbol `$;$' to make sure that $\mu$ is a variable, but not a random variable. In this paper, we assume a finite number of realisations of each random variable.
\par
The full joint distribution implies that we can no longer discuss only the probabilities of individual formulas, but they are derived from data. For example, the probability of $\alpha\in L$ is calculated as follows.
{\small
\begin{align}
p(\alpha)&=\sum_{m}\sum_{d}p(\alpha,m,d)=\sum_{m}p(\alpha|m)\sum_{d}p(m|d)p(d)\label{eq:4}
\end{align}
}
Here, the sum rule of probability theory is applied in the first equation, and Equation (\ref{eq:3}) in the second equation.
\begin{proposition}\label{negation}
Let $p(L,M,D;\mu)$ be a generative reasoning model. For all $\alpha\in L$, $p(\alpha=0)=p(\neg\alpha=1)$ holds.
\end{proposition}
\begin{proof}
For all models $m$, $\alpha$ is false in $m$ if and only if $\lnot\alpha$ is true in $m$. Thus, $\ms{\alpha=0}=\ms{\lnot\alpha=1}$ is the case. Therefore,
\begin{align*}
p(\alpha=0)&\textstyle{=\sum_{m}p(\alpha=0|m)p(m)}\\
&\textstyle{=\sum_{m}\mu^{\ms{\alpha=0}_{m}}(1-\mu)^{1-\ms{\alpha=0}_{m}}p(m)}\\
&\textstyle{=\sum_{m}\mu^{\ms{\lnot\alpha=1}_{m}}(1-\mu)^{1-\ms{\lnot\alpha=1}_{m}}p(m)}\\
&\textstyle{=\sum_{m}p(\lnot\alpha=1|m)p(m)=p(\lnot \alpha=1)}.
\end{align*}
This holds regardless of the value of $\mu$.
\end{proof}
Hereinafter, we replace $\alpha=0$ by $\lnot\alpha=1$ and abbreviate $\lnot\alpha=1$ to $\lnot\alpha$. We also abbreviate $M=m_{n}$ to $m_{n}$ and $D=d_{k}$ to $d_{k}$. 
\par
The hierarchy shown in Figure \ref{fig:hierarchy} illustrates Equation (\ref{eq:4}). The top layer of the hierarchy is a probability distribution over data, the middle layer is a probability distribution over states of the world, often referred to as models in formal logic, and the bottom layer is a probability distribution over a logical formula $\alpha$. A darker colour indicates a higher probability. Each element of a lower layer is an abstraction, selective ignorance, of the linked element of the upper.
\begin{figure}[t]
\begin{center}
 \includegraphics[scale=0.22]{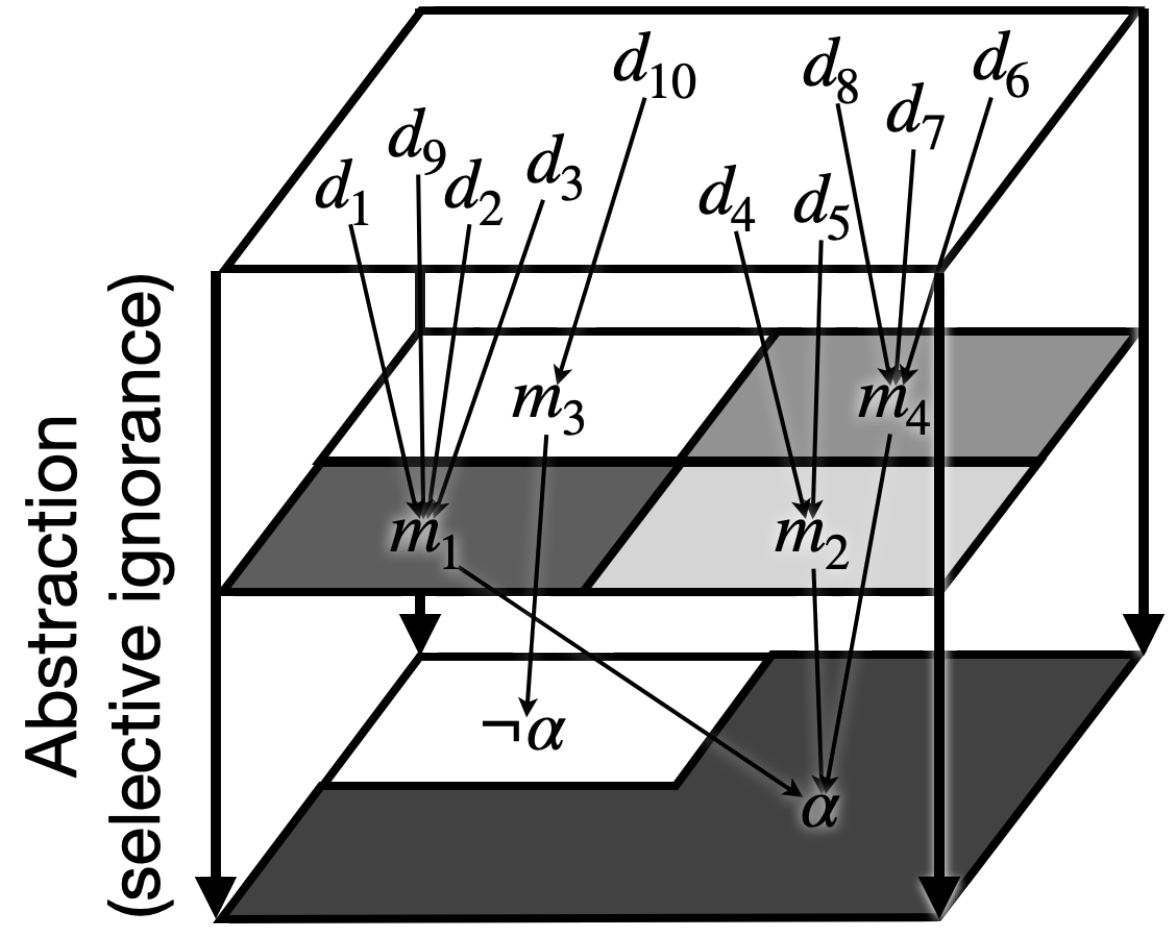}
  \caption{A schematic of how the probability distribution over data determines the probability distribution over logical formulas. For simplicity, an arrow is omitted if the formula at the end of the arrow is false in the model at the start of the arrow and if the model at the end of the arrow is not supported by the data at the start of the arrow.}
  \label{fig:hierarchy}
  \end{center}
\end{figure}
\begin{example}
Let $L$ be a propositional language built with two symbols, $rain$ and $wet$, meaning `rain falls' and `the road gets wet,' respectively. Let $m_{n}(1\leq n\leq 4)$ be the models of $L$ and $d_{k}(1\leq k\leq 10)$ be data about rain and road conditions. Table \ref{tab:hierarchy} shows which data support which models and which models specify which states of the world. The probability of $rain\to wet$ can be calculated using Equation (\ref{eq:4}) as follows.
\begin{align*}
&\textstyle{p(rain\to wet)}\\
&\textstyle{=\sum_{n=1}^{4}p(rain\to wet|m_{n})\sum_{k=1}^{10}p(m_{n}|d_{k})p(d_{k})}\\
&\textstyle{=\mu\sum_{k=1}^{10}p(m_{1}|d_{k})\frac{1}{10}+\mu\sum_{k=1}^{10}p(m_{2}|d_{k})\frac{1}{10}}\\
&\textstyle{~~~~~+(1-\mu)\sum_{k=1}^{10}p(m_{3}|d_{k})\frac{1}{10}+\mu\sum_{k=1}^{10}p(m_{4}|d_{k})\frac{1}{10}}\\
&\textstyle{=\frac{4}{10}\mu+\frac{2}{10}\mu+\frac{1}{10}(1-\mu)+\frac{3}{10}\mu=\frac{1}{10}+\frac{8}{10}\mu}
\end{align*}
Therefore, $p(rain\to wet)=9/10$ when $\mu=1$ or $\mu\to 1$, i.e., $\mu$ approaching 1. Figure \ref{fig:hierarchy} illustrates the calculation and visualises how the probability of $rain\to wet (=\alpha)$ is derived from data.
\end{example}
\begin{table}[t]
\centering
\caption{An example of Figure \ref{fig:hierarchy}. From the left, each column shows data, models and the likelihood of the formula.}
\label{tab:hierarchy}
\begin{tabular}{c|ccc|c}
$D$ & $M$ & $rain$ & $wet$  & $p(rain\to wet|M)$\\\hline
$d_{1},d_{2},d_{3},d_{9}$ & $m_{1}$ & 0 & 0 & $\mu$\\
$d_{4},d_{5}$ & $m_{2}$ & 0 & 1 & $\mu$\\
$d_{10}$ & $m_{3}$ & 1 & 0 & $1-\mu$\\
$d_{6},d_{7},d_{8}$ & $m_{4}$ & 1 & 1 & $\mu$
\end{tabular}
\end{table}
\section{Correctness}
\subsection{Logical reasoning}\label{sec:consistency}
In the previous section, we saw that the probabilities of models and formulas are derived from data. As a result, the probability of a model without support from data and the probability of a formula satisfied only by such models both turn out to be zero. We refer to such models and formulas as being impossible.
\begin{definition}[Possibility]
Let $m$ be a model associated with $L$. $m$ is possible if $p(m)\neq 0$ and impossible otherwise.
\end{definition}
For $\Delta\subseteq L$, we use symbol $\pms{\Delta}$ to denote the set of all the possible models of $\Delta$, i.e., $\pms{\Delta}=\{m\in\ms{\Delta}|p(m)\neq 0\}$. We also use symbol $\pms{\Delta}_{m}$ such that $\pms{\Delta}_{m}=1$ if $m\in\pms{\Delta}$ and $\pms{\Delta}_{m}=0$ otherwise. Obviously, $\pms{\Delta}\subseteq\ms{\Delta}$, for all $\Delta\subseteq L$, and $\pms{\Delta}=\ms{\Delta}$ if all models are possible. If $\Delta$ is inconsistent, $\pms{\Delta}=\ms{\Delta}=\emptyset$. If $\Delta$ is an empty set or if it only includes tautologies then every model satisfies all the formulas in the possibly empty $\Delta$, and thus $\ms{\Delta}$ includes all the models.
\par
In this section, we look at generative reasoning models with $\mu=1$, $p(L,M,D;\mu=1)$, for reasoning from a consistent source of information.  The following theorem relates the probability of a formula to the probability of its models.
\begin{theorem}\label{thrm:consistency}
Let $p(L,M,D;\mu=1)$ be a generative reasoning model, and $\alpha\in L$ and $\Delta\subseteq L$ such that $\ms{\Delta}=\pms{\Delta}$.
\begin{align*}
p(\alpha|\Delta)=
\begin{cases}
\displaystyle{\frac{\sum_{m\in\ms{\Delta}\cap\ms{\alpha}}p(m)}{\sum_{m\in\ms{\Delta}}p(m)}}&\text{if }\ms{\Delta}\neq\emptyset\\
\text{undefined}&\text{otherwise}
\end{cases}
\end{align*}
\end{theorem}
\begin{proof}
Let $|\Delta|$ denote the cardinality of $\Delta$. Dividing models into the ones satisfying all the formulas in $\Delta$ and the others, we have
{\small
\begin{align*}
&p(\alpha|\Delta)=\frac{\sum_{m}p(\alpha|m)p(\Delta|m)p(m)}{\sum_{m}p(\Delta|m)p(m)}\\
&=\frac{\displaystyle{\sum_{m\in\llbracket\Delta\rrbracket}p(m)p(\alpha|m)\mu^{|\Delta|}+\sum_{m\notin\llbracket\Delta\rrbracket}p(m)p(\alpha|m)p(\Delta|m)}}{\displaystyle{\sum_{m\in\llbracket\Delta\rrbracket}p(m)\mu^{|\Delta|}+\sum_{m\notin\llbracket\Delta\rrbracket}p(m)p(\Delta|m)}}.
\end{align*}
}
By definition, $p(\Delta|m)=\prod_{\beta\in\Delta}p(\beta|m)=\prod_{\beta\in\Delta}\mu^{\ms{\beta}_{m}}(1-\mu)^{1-{\ms{\beta}_{m}}}$. For all $m\notin\llbracket\Delta\rrbracket$, there is $\beta\in\Delta$ such that $\ms{\beta}_{m}=0$. Therefore, $p(\Delta|m)=0$ when $\mu=1$, for all $m\notin\llbracket\Delta\rrbracket$. We thus have
{\small
\begin{align*}
p(\alpha|\Delta)=\frac{\displaystyle{\sum_{m\in\llbracket\Delta\rrbracket}p(m)p(\alpha|m)1^{|\Delta|}}}{\displaystyle{\sum_{m\in\llbracket\Delta\rrbracket}p(m)1^{|\Delta|}}}=\frac{\displaystyle{\sum_{m\in\llbracket\Delta\rrbracket}p(m)1^{\ms{\alpha}_{m}}0^{1-\ms{\alpha}_{m}}}}{\displaystyle{\sum_{m\in\llbracket\Delta\rrbracket}p(m)}}.
\end{align*}
}
We obtain the theorem, since $1^{\ms{\alpha}_{m}}0^{1-\ms{\alpha}_{m}}=1^{1}0^{0}=1$ if $m\in\ms{\alpha}$ and $1^{\ms{\alpha}_{m}}0^{1-\ms{\alpha}_{m}}=1^{0}0^{1}=0$ if $m\notin\ms{\alpha}$. In addition, if $\ms{\Delta}=\emptyset$ then $p(\alpha|\Delta)$ is undefined due to division by zero.
\end{proof}
Recall that a formula $\alpha$ is a logical consequence of a set $\Delta$ of formulas, denoted by $\Delta\cent\alpha$, in classical logic iff (if and only if) $\alpha$ is true in every model in which $\Delta$ is true, i.e., $\ms{\Delta}\subseteq\ms{\alpha}$. The following Corollary shows the relationship between the generative reasoning model $p(L,M,D;\mu=1)$ and the classical consequence relation $\cent$.
\begin{corollary}\label{cor:consistent_reasoning}
Let $p(L,M,D;\mu=1)$ be a generative reasoning model, and $\alpha\in L$ and $\Delta\subseteq L$ such that $\ms{\Delta}=\pms{\Delta}$ and $\ms{\Delta}\neq\emptyset$. $p(\alpha|\Delta)=1$ iff $\Delta\cent\alpha$.
\end{corollary}
\begin{proof}
By the assumptions $\ms{\Delta}=\pms{\Delta}$ and $\ms{\Delta}\neq\emptyset$, $p(m)$ is non zero, for all $m$ in the non-empty set $\ms{\Delta}$. The assumptions thus prohibit a division by zero in Theorem \ref{thrm:consistency}. Therefore, $\frac{\sum_{m\in\ms{\Delta}\cap\ms{\alpha}}p(m)}{\sum_{m\in\ms{\Delta}}p(m)}=1$ iff $\ms{\alpha}\supseteq\ms{\Delta}$, i.e., $\Delta\models\alpha$.
\end{proof}
The following example shows the importance of the assumptions of $\ms{\Delta}=\pms{\Delta}$ and $\ms{\Delta}\neq\emptyset$ in Corollary \ref{cor:consistent_reasoning}.
\begin{table}[t]
\begin{center}
\caption{Some inconsistencies between generative reasoning and classical reasoning.}
\label{ex:classical_consequence}
\scalebox{0.8}{
\begin{tabular}{l|l|l}\label{table:consistent}
Generative reasoning & Classical reasoning & Rationale\\\hline
$p(wet|rain,\lnot rain)\neq1$ & $rain,\lnot rain\cent wet$ & $\ms{rain,\lnot rain}=\emptyset$\\
$p(wet|rain)=1$ & $rain\not\cent wet$ & $\ms{rain}\neq\pms{rain}$\\
$p(\lnot rain\lor wet)=1$ & $\not\cent \lnot rain\lor wet$ & $\ms{\emptyset}\neq\pms{\emptyset}$
\end{tabular}
}
\end{center}
\end{table}
\begin{example}
Suppose that the probability distribution in Table \ref{tab:hierarchy} is given by $p(M)=(m_{1},m_{2},m_{3},m_{4})=(0.5,0.2,0,0.3)$. Table \ref{ex:classical_consequence} exemplifies differences between the generative reasoning and classical consequence relation. The last column explains why the generative reasoning is inconsistent with the classical consequence. In particular, the rationale of the last example comes from the fact that Theorem \ref{thrm:consistency} explains $p(\lnot rain\lor wet)$ as $p(\lnot rain\lor wet|\emptyset)$.
\begin{align*}
&p(\lnot rain\lor wet)=p(\lnot rain\lor wet|\emptyset)\\
&=\frac{\sum_{m\in\ms{\emptyset}\cap\ms{\lnot rain\lor wet}}p(m)}{\sum_{m\in\ms{\emptyset}}p(m)}=\frac{\sum_{m\in\ms{\lnot rain\lor wet}}p(m)}{\sum_{m}p(m)}\\
&\textstyle{=\sum_{m\in\ms{\lnot rain\lor wet}}p(m)=1.}
\end{align*}
Here, $\ms{\emptyset}=\{m_{1},m_{2},m_{3},m_{4}\}$ but $\pms{\emptyset}=\{m_{1},m_{2},m_{4}\}$.
\end{example}
\par
Figure \ref{fig:types} illustrates the assumptions of $\ms{\Delta}=\pms{\Delta}$ and $\ms{\Delta}\neq\emptyset$ for reasoning of $\alpha\in L$ from $\Delta\subseteq L$ using the generative reasoning model $p(L,M,D;\mu=1)$. Both $\alpha$ and $\Delta$ are consistent, since there is at least one model satisfying $\alpha$ and all the formulas in $\Delta$, i.e., $\ms{\alpha}\neq\emptyset$ and $\ms{\Delta}\neq\emptyset$. Such models are highlighted on the middle layer in blue and green, respectively. Figure \ref{fig:types} also shows that every model satisfying all the formulas in $\Delta$ is possible, since there is at least one  datum that supports each model of $\Delta$, i.e., $\ms{\Delta}=\pms{\Delta}$.
\begin{figure*}[t]
\centering
\includegraphics[scale=0.27]{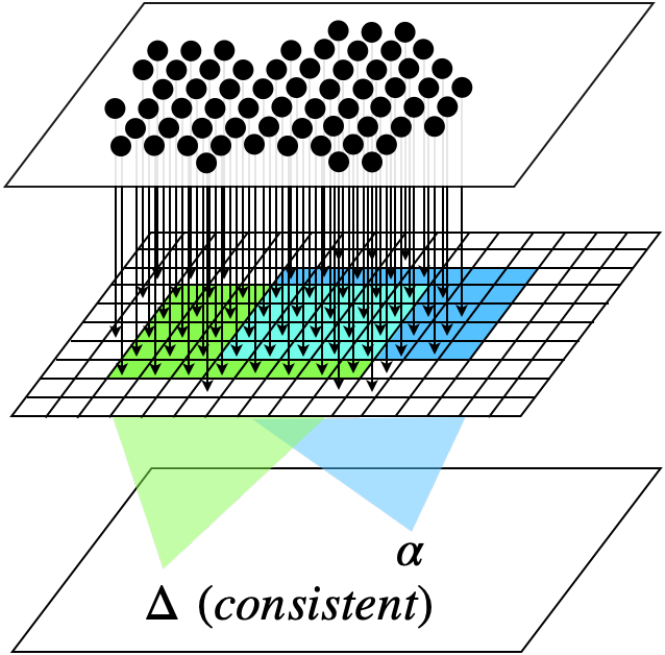}
\includegraphics[scale=0.27]{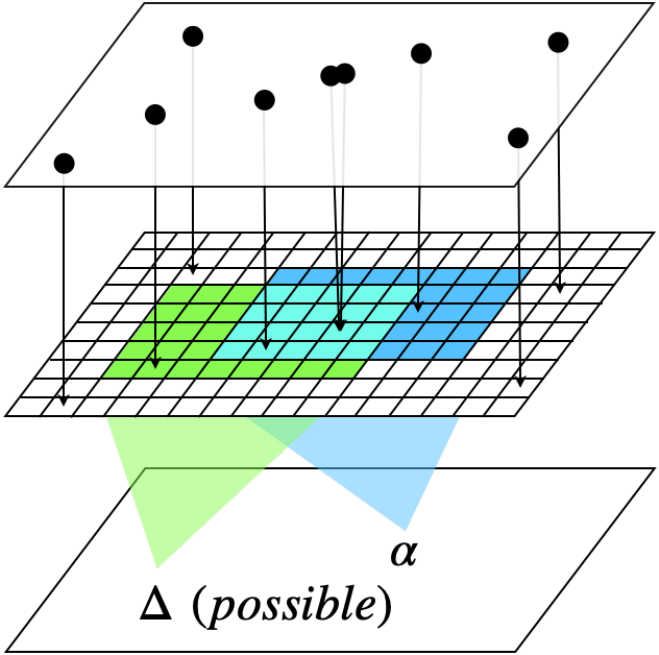}
\includegraphics[scale=0.27]{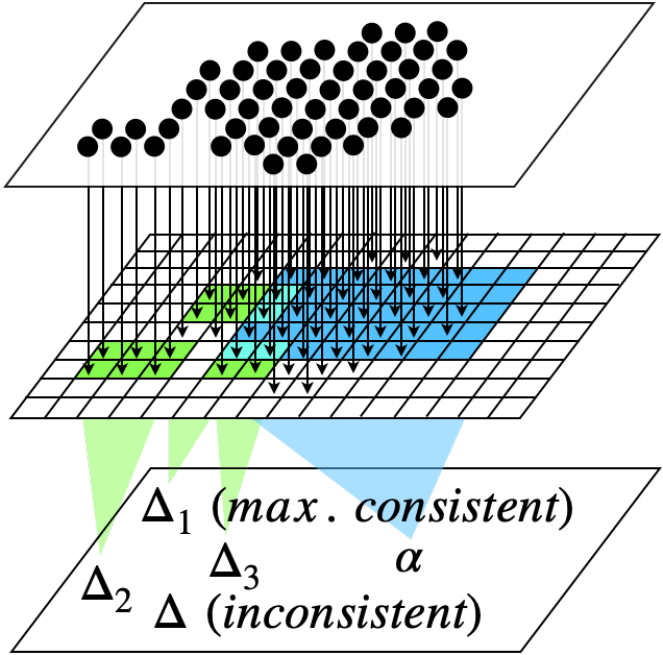}
\includegraphics[scale=0.27]{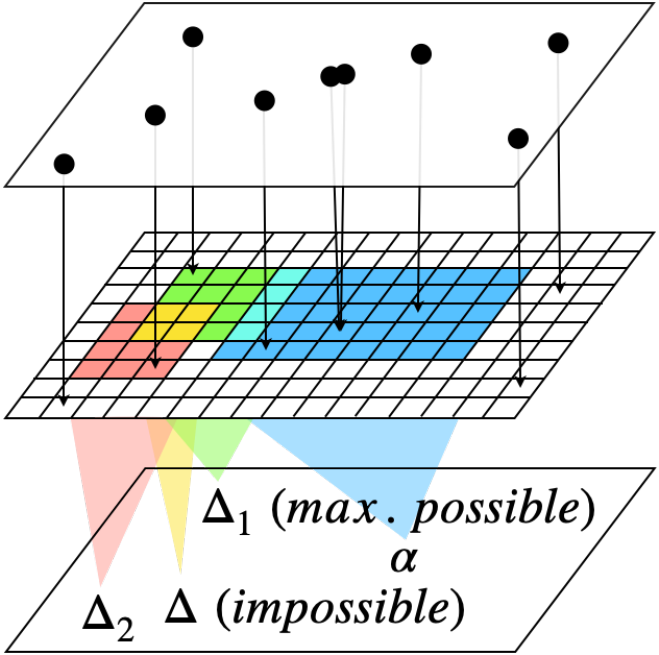}
\caption{The left two graphs illustrate reasoning of $\alpha\in L$ from $\Delta\subseteq L$ using $p(L,M,D;\mu=1)$. The leftmost shows the assumptions of $\ms{\Delta}=\pms{\Delta}$ and $\ms{\Delta}\neq\emptyset$. Each arrow from a datum to model, denoted respectively by a black circle on the top layer and a cell on the middle layer, represents that the datum supports the model. Each model with an incoming arrow thus has a non-zero probability. A model is coloured in green (resp. blue) if all the formulas in $\Delta$ are (resp. $\alpha$) true in the model. The second shows the assumption of $\pms{\Delta}\neq\emptyset$. The right two graphs illustrate reasoning of $\alpha\in L$ from $\Delta\subseteq L$ using $p(L,M,D;\mu\to1)$. The third shows the assumption of $\ams{\Delta}=\pams{\Delta}$. $\Delta_{1}$, $\Delta_{2}$ and $\Delta_{3}$ are the cardinality-maximal consistent subsets of $\Delta$. The rightmost shows no assumption. $\Delta_{1}$ and $\Delta_{2}$ are the cardinality-maximal possible subsets of $\Delta$.}
\label{fig:types}
\end{figure*}

\subsection{Empirical reasoning}\label{sec:possibility}
Theorem \ref{thrm:consistency} and Corollary \ref{cor:consistent_reasoning} depend on the assumption of $\ms{\Delta}=\pms{\Delta}$. In this section, we cancel the assumption to fully generalise our discussions in Section \ref{sec:consistency}. The following theorem relates the probability of a formula to the probability of its possible models.
\begin{theorem}\label{thrm:possibility}
Let $p(L,M,D;\mu=1)$ be a generative reasoning model, and $\alpha\in L$ and $\Delta\subseteq L$.
\begin{eqnarray*}
p(\alpha|\Delta)=
\begin{cases}
\displaystyle{\frac{\sum_{m\in\pms{\Delta}\cap\pms{\alpha}}p(m)}{\sum_{m\in\pms{\Delta}}p(m)}}&\text{if }\pms{\Delta}\neq\emptyset\\
\text{undefined}&\text{otherwise}
\end{cases}
\end{eqnarray*}
\end{theorem}
\begin{proof}
$p(m)=0$, for all $m\in\ms{\Delta}\setminus\pms{\Delta}$ and $m\in\ms{\alpha}\setminus\pms{\alpha}$. From Theorem \ref{thrm:consistency}, we thus have 
\begin{align*}
\displaystyle{\frac{\sum_{m\in\ms{\Delta}\cap\ms{\alpha}}p(m)}{\sum_{m\in\ms{\Delta}}p(m)}}=\displaystyle{\frac{\sum_{m\in\pms{\Delta}\cap\pms{\alpha}}p(m)}{\sum_{m\in\pms{\Delta}}p(m)}}.
\end{align*}
The condition of $\ms{\Delta}\neq\emptyset$ should be replaced by $\pms{\Delta}\neq\emptyset$, since there is a possibility of $\ms{\Delta}\neq\emptyset$ and $\pms{\Delta}=\emptyset$. Given the condition of $\ms{\Delta}\neq\emptyset$, this causes a probability undefined due to a division by zero.
\end{proof}
In Section \ref{sec:consistency}, we used the classical consequence relation in Corollary \ref{cor:consistent_reasoning} for a logical characterisation of Theorem \ref{thrm:consistency}. In this section, we define an alternative consequence relation for a logical characterisation of Theorem \ref{thrm:possibility}.
\begin{definition}[Empirical consequence]
Let $\Delta\subseteq L$ and $\alpha\in L$. $\alpha$ is an empirical consequence of $\Delta$, denoted by $\Delta\ent\alpha$, if $\pms{\Delta}\subseteq\pms{\alpha}$.
\end{definition}
\begin{proposition}
Let $\Delta\subseteq L$ and $\alpha\in L$. If $\Delta\cent\alpha$ then $\Delta\ent\alpha$, but not vice versa.
\end{proposition}
\begin{proof}
($\Rightarrow$) $\Delta\ent\alpha$ iff $\pms{\Delta}\subseteq\pms{\alpha}$ where $\pms{\Delta}=\{m\in\ms{\Delta}|p(m)\neq0\}$ and $\pms{\alpha}=\{m\in\ms{\alpha}|p(m)\neq0\}$. $\ms{\Delta}\subseteq\ms{\alpha}$ implies $\pms{\Delta}\subseteq\pms{\alpha}$, since $\ms{\Delta}\setminus X\subseteq\ms{\alpha}\setminus X$, for all sets $X$. ($\Leftarrow$) Suppose $\Delta$, $\alpha$ and $m$ such that $\ms{\Delta}=\ms{\alpha}\cup\{m\}$ and $p(m)=0$. Then, $\Delta\ent\alpha$, but $\Delta\not\cent\alpha$.
\end{proof}
The following Corollary shows the relationship between the generative reasoning model $p(L,M,D;\mu=1)$ and the empirical consequence relation $\ent$.
\begin{corollary}\label{cor:possible_reasoning}
Let $p(L,M,D;\mu=1)$ be a generative reasoning model, and $\alpha\in L$ and $\Delta\subseteq L$ such that $\pms{\Delta}\neq\emptyset$. $p(\alpha|\Delta)=1$ iff $\Delta\ent\alpha$.
\end{corollary}
\begin{proof}
$\Delta\ent\alpha$ iff $\pms{\Delta}\subseteq\pms{\alpha}$. $p(m)\neq0$, for all $m\in\pms{\Delta}$. Thus, from Theorem \ref{thrm:possibility}, $p(\alpha|\Delta)=1$ iff $\pms{\Delta}\subseteq\pms{\alpha}$.
\end{proof}
Note that Theorem \ref{thrm:possibility} and Corollary \ref{cor:possible_reasoning} no longer depend on the assumption of $\ms{\Delta}=\pms{\Delta}$ required in Theorem \ref{thrm:consistency} and Corollary \ref{cor:consistent_reasoning}. Figure \ref{fig:types} illustrates the assumption of $\pms{\Delta}\neq\emptyset$ for reasoning of $\alpha\in L$ from $\Delta\subseteq L$ using the generative reasoning model $p(L, M, D; \mu=1)$. It shows that both $\alpha$ and $\Delta$ are consistent, i.e., $\ms{\alpha}\neq\emptyset$ and $\ms{\Delta}\neq\emptyset$, since there is at least one model for both $\alpha$ and $\Delta$ satisfying the formulas. It also shows that $\Delta$ and $\alpha$ are possible, i.e., $\pms{\Delta}\neq\emptyset$ and $\pms{\alpha}\neq\emptyset$, since there is at least one model for both $\Delta$ and $\alpha$ supported by data.
\subsection{Logical reasoning in inconsistency}\label{sec:paraconsistency}
\begin{figure}[t]
 \begin{center}
\includegraphics[scale=0.2]{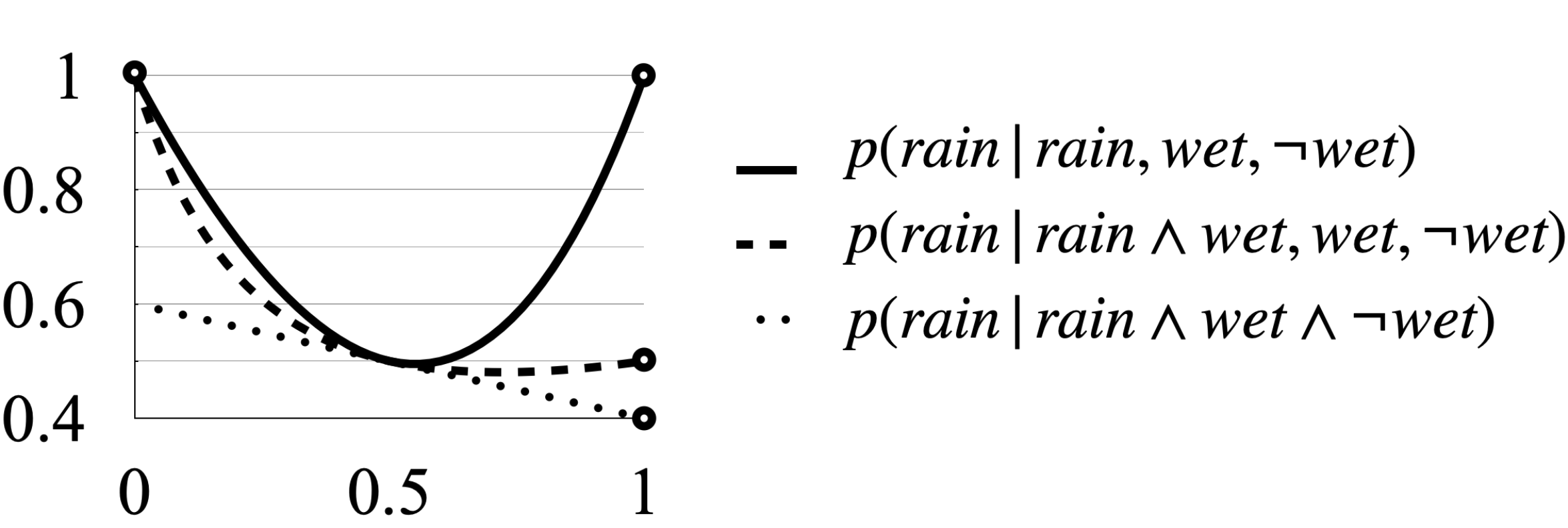}
 \end{center}
\caption{Three examples of reasoning from inconsistency. The probability versus $\mu$.}
\label{fig:ex_limit}
\end{figure}
Theorem \ref{thrm:consistency} and Corollary \ref{cor:consistent_reasoning} assume $\ms{\Delta}\neq\emptyset$ in practice. The conditional probability $p(\alpha|\Delta)$ is undefined otherwise. This section aims to cancel the assumption to fully generalise our discussions in Section \ref{sec:consistency} so that we can reason also from an inconsistent source of information. To this end, we look at the generative reasoning model $p(L,M,D;\mu\to 1)$, rather than $p(L,M,D;\mu=1)$, where $\mu\to 1$ represents $\mu$ approaching one, i.e., $\lim_{\mu\to 1}$. The following example shows the intuition of how the limit works and how it naturally generalises reasoning regardless of the consistency of its premises.
\begin{example}
Consider the three conditional probabilities given different inconsistent premises shown in Figure \ref{fig:ex_limit}. Suppose that the probability distribution in Table \ref{tab:hierarchy} is given by $p(M)=(m_{1},m_{2},m_{3},m_{4})=(0.4,0.2,0.1,0.3)$. The conditional probability shown on the top right is expanded as follows, where $rain$ and $wet$ are abbreviated to $r$ and $w$.
{\small
\begin{align*}
&p(r|r,w,\lnot w)=\frac{\sum_{m}p(r|m)^{2}p(w|m)p(\lnot w|m)p(m)}{\sum_{m}p(r|m)p(w|m)p(\lnot w|m)p(m)}\\
&=\frac{(p(m_{1})+p(m_{2}))\mu(1-\mu)^{3}+(p(m_{3})+p(m_{4}))\mu^{3}(1-\mu)}{(p(m_{1})+p(m_{2}))\mu(1-\mu)^{2}+(p(m_{3})+p(m_{4}))\mu^{2}(1-\mu)}\\
&=\frac{0.6\mu(1-\mu)^{3}+0.4\mu^{3}(1-\mu)}{0.6\mu(1-\mu)^{2}+0.4\mu^{2}(1-\mu)}
\end{align*}
}
The graph with the solid line in Figure \ref{fig:ex_limit} shows $p(rain|rain,wet,\lnot wet)$ given different $\mu$ values. The graph also includes the other two conditional probabilities calculated in the same manner. Each of the open circles represents an undefined value. This means that no substitution gives a probability, even though the curve approaches a certain probability. The certain probability can only be obtained by the use of limit. Indeed, given $\mu\to1$, the three conditional probabilities turn out to be 1, 0.5 and 0.4, respectively.
\end{example}
Everything is entailed from an inconsistent set of formulas in formal logic. The use of limit is a reasonable alternative, since it allows us to consider what if there is a very tiny chance of the formula being true. The mathematical correctness of the solution can be shown using maximal consistent sets and maximal possible sets.
\begin{definition}[Maximal consistent sets]
Let $S,\Delta\subseteq L$. $S\subseteq\Delta$ is a maximal consistent subset of $\Delta$ if $\ms{S}\neq\emptyset$ and $\ms{S\cup\{\alpha\}}=\emptyset$, for all $\alpha\in\Delta\setminus S$.
\end{definition}
We refer to a maximal consistent subset as a cardinality-maximal consistent subset when the set has the maximum cardinality. We use symbol $MCS(\Delta)$ to denote the set of the cardinality-maximal consistent subsets of $\Delta\subseteq L$. We use symbol $\ams{\Delta}$ to denote the set of the models of the cardinality-maximal consistent subsets of $\Delta$, i.e., $\ams{\Delta}=\bigcup_{S\in MCS(\Delta)}\ms{S}$.
\begin{example}[Cardinality-maximal consistent sets]\label{ex:MCS}
Consider the model distribution shown in Figure \ref{fig:ex_limit}. Consider $\Delta=\{$ $rain$, $wet$, $rain\to wet$, $\lnot wet\}$. It gives the following three maximal consistent subsets of $\Delta$: $S_{1}=\{rain,wet,rain\to wet\}$, $S_{2}=\{rain,\lnot wet\}$ and $S_{3}=\{rain\to wet,\lnot wet\}$. Only $S_{1}$ is the cardinality-maximal consistent subset of $\Delta$, i.e., $MCS(\Delta)=\{S_{1}\}$. Therefore, $\ams{\Delta}=\bigcup_{S\in MCS(\Delta)}\ms{S}=\ms{S_{1}}=\{m_{4}\}$.
\end{example}
\begin{definition}[Maximal possible sets]
Let $S,\Delta\subseteq L$. $S\subseteq\Delta$ is a maximal possible subset of $\Delta$ if $\pms{S}\neq\emptyset$ and $\pms{S\cup\{\alpha\}}=\emptyset$, for all $\alpha\in\Delta\setminus S$.
\end{definition}
Similarly, we refer to a maximal possible subset as a cardinality-maximal possible subset when the set has the maximum cardinality. We use symbol $MPS(\Delta)$ to denote the set of the cardinality-maximal possible subsets of $\Delta\subseteq L$. We use symbol $\pams{\Delta}$ to denote the set of possible models of the cardinality-maximal possible subsets of $\Delta$, i.e., $\pams{\Delta}=\bigcup_{S\in MPS(\Delta)}\pms{S}$.
\begin{example}[Cardinality-maximal possible sets]\label{ex:MPS}
Suppose that the probability distribution in Table \ref{tab:hierarchy} is given by $p(M)=(m_{1},m_{2},m_{3},m_{4})=(0.9,0.1,0,0)$. Consider $\Delta=\{$ $rain$, $wet$, $rain\to wet,\lnot wet\}$. It gives the following two maximal possible subsets of $\Delta$: $S_{1}=\{wet,rain\to wet\}$ and $S_{2}=\{rain\to wet,\lnot wet\}$. Both $S_{1}$ and $S_{2}$ are the cardinality-maximal possible subsets of $\Delta$, i.e., $MPS(\Delta)=\{S_{1},S_{2}\}$. Only $m_{2}$ is the possible model of $S_{1}$ and $m_{1}$ is the possible model of $S_{2}$. Namely, $\pms{S_{1}}=\{m_{2}\}$ and $\pms{S_{2}}=\{m_{1}\}$. Therefore, $\pams{\Delta}=\bigcup_{S\in MPS(\Delta)}\pms{S}=\{m_{1},m_{2}\}$.
\end{example}
Obviously, $\ams{\Delta}=\ms{\Delta}$ if there is a model of $\Delta$, i.e., $\ms{\Delta}\neq\emptyset$. Similarly, $\pams{\Delta}=\pms{\Delta}$ if there is a possible model of $\Delta$, i.e., $\pms{\Delta}\neq\emptyset$. Note that if $\Delta$ is an empty set or $\Delta$ only includes tautologies then every model satisfies all the formulas in the possibly empty $\Delta$. $\ms{\Delta}$ is thus the set of all models, and therefore $\ms{\Delta}\neq\emptyset$. Moreover, $\pms{\Delta}\neq\emptyset$, since $p(M)$ is a probability distribution, and thus, there is at least one model $m$ such that $p(m)\neq 0$.
\begin{lemma}\label{lemma:paraconsistency}
$\ams{\Delta}\neq\emptyset$, for all $\Delta\subseteq L$.
\end{lemma}
\begin{proof}
By definition, $\ams{\Delta}=\bigcup_{S\in MCS(\Delta)}\ms{S}$. For all $\Delta\subseteq L$, the empty set $\emptyset$ is a consistent subset of $\Delta$. Thus, there is at least one (possibly empty) maximal consistent subset of $\Delta$. Namely, $MCS(\Delta)\neq\emptyset$. For all $S\in MCS(\Delta)$, since $S$ is consistent, there is at least one model satisfying all the elements of (possibly empty) $S$. Namely, $\ms{S}\neq\emptyset$. This guarantees $\bigcup_{S\in MCS(\Delta)}\ms{S}\neq\emptyset$.
\end{proof}
\begin{example}
Let $\Delta_{1}=\{\alpha,\lnot\alpha\}$ and $\Delta_{2}=\{\alpha\land\lnot\alpha\}$, for $\alpha\in L$. $\ams{\Delta_{1}}=$ $\bigcup_{S\in MCS(\Delta_{1})}\ms{S}=$ $\bigcup_{S\in \{\{\alpha\},\{\lnot\alpha\}\}}\ms{S}=$ $\ms{\alpha}\cup\ms{\lnot\alpha}={\cal M}$. $\ams{\Delta_{2}}$ $=\bigcup_{S\in MCS(\Delta_{2})}\ms{S}=$ $\bigcup_{S\in \{\emptyset\}}\ms{S}=$ $\ms{\emptyset}={\cal M}$.
Here, ${\cal M}$ denotes all the models associated with $L$.
\end{example}
\begin{lemma}\label{lemma:parapossibility}
$\pams{\Delta}\neq\emptyset$, for all $\Delta\subseteq L$.
\end{lemma}
\begin{proof}
Same as Lemma \ref{lemma:paraconsistency}. Use $MPS(\Delta)$ and $\pms{\Delta}$.
\end{proof}
\begin{example}
Suppose that the probability distribution in Table \ref{tab:hierarchy} is given by $p(M)=(m_{1},m_{2},m_{3},m_{4})=(0.5,0.2,0,0.3)$.
Let $\Delta_{1}=\{rain,\lnot rain\}$ and $\Delta_{2}=\{rain\land\lnot rain\}$. $\pams{\Delta_{1}}$ $=\bigcup_{S\in MPS(\Delta_{1})}\pms{S}$ $=\pms{rain}\cup\pms{\lnot rain}$ $=\{m_{1},m_{2},m_{4}\}$. $\pams{\Delta_{2}}$ $=\bigcup_{S\in MPS(\Delta_{2})}\pms{S}$ $=\pms{\emptyset}$ $=\{m_{1},m_{2},m_{4}\}$
\end{example}
\par
The generative reasoning model $p(L,M,D;\mu\to1)$ has the following property.
\begin{theorem}\label{thrm:paraconsistency}
Let $p(L,M,D;\mu\to1)$ be a generative reasoning model, and $\alpha\in L$ and $\Delta\subseteq L$ such that $\ams{\Delta}=\pams{\Delta}$.
\begin{align*}
p(\alpha|\Delta)=\frac{\sum_{m\in\ams{\Delta}\cap\ms{\alpha}}p(m)}{\sum_{m\in\ams{\Delta}}p(m)}
\end{align*}
\end{theorem}
\begin{proof}
We use symbol $|\Delta|$ to denote the number of formulas in $\Delta$ and symbol $|\Delta|_{m}$ to denote the number of formulas in $\Delta$ that are true in $m$, i.e., $|\Delta|_{m}=\sum_{\beta\in\Delta}\m{\beta}_{m}$. Dividing models into $\ams{\Delta}$ and the others, we have
{\small
\begin{align*}
&p(\alpha|\Delta)=\lim_{\mu\rightarrow 1}\frac{\sum_{m}p(\alpha|m)p(m)p(\Delta|m)}{\sum_{m}p(m)p(\Delta|m)}=\lim_{\mu\to 1}\\
&\frac{\displaystyle{\sum_{\hat{m}\in\ams{\Delta}}p(\alpha|\hat{m})p(\hat{m})p(\Delta|\hat{m})+\sum_{m\notin\ams{\Delta}}p(\alpha|m)p(m)p(\Delta|m)}}{\displaystyle{\sum_{\hat{m}\in\ams{\Delta}}p(\hat{m})p(\Delta|\hat{m})+\sum_{m\notin\ams{\Delta}}p(m)p(\Delta|m)}}.
\end{align*}
}
Now, $p(\Delta|m)$ can be developed as follows, for all $m$ (regardless of the membership of $\ams{\Delta}$).
{\small
\begin{align*}
&p(\Delta|m)=\prod_{\beta\in\Delta}p(\beta|m)=\prod_{\beta\in\Delta}\mu^{\ms{\beta}_{m}}(1-\mu)^{1-\ms{\beta}_{m}}\\
&=\mu^{\sum_{\beta\in\Delta}\ms{\beta}_{m}}(1-\mu)^{\sum_{\beta\in\Delta}(1-\ms{\beta}_{m})}=\mu^{|\Delta|_{m}}(1-\mu)^{|\Delta|-|\Delta|_{m}}
\end{align*}
}
Therefore, $p(\alpha|\Delta)=\lim_{\mu\rightarrow 1}\frac{W+X}{Y+Z}$ where
{\small
\begin{align*}
W&=\textstyle{\sum_{\hat{m}\in\ams{\Delta}}p(\alpha|\hat{m})p(\hat{m})\mu^{|\Delta|_{\hat{m}}}(1-\mu)^{|\Delta|-|\Delta|_{\hat{m}}}}\\
X&=\textstyle{\sum_{m\notin\ams{\Delta}}p(\alpha|m)p(m)\mu^{|\Delta|_{m}}(1-\mu)^{|\Delta|-|\Delta|_{m}}}\\
Y&=\textstyle{\sum_{\hat{m}\in\ams{\Delta}}p(\hat{m})\mu^{|\Delta|_{\hat{m}}}(1-\mu)^{|\Delta|-|\Delta|_{\hat{m}}}}\\
Z&=\textstyle{\sum_{m\notin\ams{\Delta}}p(m)\mu^{|\Delta|_{m}}(1-\mu)^{|\Delta|-|\Delta|_{m}}}.
\end{align*}
}
$\ams{\Delta}\neq\emptyset$ from Lemma \ref{lemma:paraconsistency}. Since $\hat{m}\in \ams{\Delta}$ is a model of a cardinality-maximal consistent subset of $\Delta$, $|\Delta|_{\hat{m}}$ has the same value, for all $\hat{m}\in\ams{\Delta}$. Therefore, the fraction can be simplified by dividing the denominator and numerator by $(1-\mu)^{|\Delta|-|\Delta|_{\hat{m}}}$. We thus have $p(\alpha|\Delta)=\lim_{\mu\rightarrow 1}\frac{W'+X'}{Y'+Z'}$ where
{\small
\begin{align*}
W'&=\textstyle{\sum_{\hat{m}\in\ams{\Delta}}p(\alpha|\hat{m})p(\hat{m})\mu^{|\Delta|_{\hat{m}}}}\\
X'&=\textstyle{\sum_{m\notin\ams{\Delta}}p(\alpha|m)p(m)\mu^{|\Delta|_{m}}(1-\mu)^{|\Delta|_{\hat{m}}-|\Delta|_{m}}}\\
Y'&=\textstyle{\sum_{\hat{m}\in\ams{\Delta}}p(\hat{m})\mu^{|\Delta|_{\hat{m}}}}\\
Z'&=\textstyle{\sum_{m\notin\ams{\Delta}}p(m)\mu^{|\Delta|_{m}}(1-\mu)^{|\Delta|_{\hat{m}}-|\Delta|_{m}}}.
\end{align*}
}
Applying the limit operation, we can cancel out $X'$ and $Z'$.
{\small
\begin{align*}
p(\alpha|\Delta)=\frac{\displaystyle{\sum_{\hat{m}\in\ams{\Delta}}p(\alpha|\hat{m})p(\hat{m})}}{\displaystyle{\sum_{\hat{m}\in\ams{\Delta}}p(\hat{m})}}=\frac{\displaystyle{\sum_{\hat{m}\in\ams{\Delta}}1^{\ms{\alpha}_{\hat{m}}}0^{1-\ms{\alpha}_{\hat{m}}}p(\hat{m})}}{\displaystyle{\sum_{\hat{m}\in\ams{\Delta}}p(\hat{m})}}.
\end{align*}
}
We have the theorem, since $1^{\ms{\alpha}_{\hat{m}}}0^{1-\ms{\alpha}_{\hat{m}}}=1^{1}0^{0}=1$ if $\hat{m}\in\ms{\alpha}$ and $1^{\ms{\alpha}_{\hat{m}}}0^{1-\ms{\alpha}_{\hat{m}}}=1^{0}0^{1}=0$ if $\hat{m}\notin\ms{\alpha}$.
\end{proof}
The following Corollary shows the relationship between the generative reasoning model $p(L,M,D;\mu\to 1)$ and the classical consequence relation with maximal consistent sets.
\begin{corollary}\label{cor:paraconsistency}
Let $p(L,M,D;\mu\to1)$ be a generative reasoning model, and $\alpha\in L$ and $\Delta\subseteq L$ such that $\ams{\Delta}=\pams{\Delta}$. $p(\alpha|\Delta)=1$ iff $S\cent\alpha$, for all cardinality-maximal consistent subsets $S$ of $\Delta$.
\end{corollary}
\begin{proof}
By the assumption of $\ams{\Delta}=\pams{\Delta}$, $p(m)$ is non zero, for all $m\in\ams{\Delta}$. From Theorem \ref{thrm:paraconsistency}, thus $p(\alpha|\Delta)=1$ iff $\ms{\alpha}\supseteq\ams{\Delta}$. Since $\ams{\Delta}=\bigcup_{S\in MCS(\Delta)}\ms{S}$, $p(\alpha|\Delta)=1$ iff $\ms{\alpha}\supseteq\bigcup_{S\in MCS(\Delta)}\ms{S}$. Namely, $p(\alpha|\Delta)=1$ iff $\ms{\alpha}\supseteq\ms{S}$, for all cardinality-maximal consistent subsets $S$ of $\Delta$.
\end{proof}
Note that Theorem \ref{thrm:paraconsistency} and Corollary \ref{cor:paraconsistency} no longer depend on the assumption of $\ms{\Delta}\neq\emptyset$ required in Section \ref{sec:consistency} for Theorem \ref{thrm:consistency} and Corollary \ref{cor:consistent_reasoning}. Figure \ref{fig:types} illustrates the assumption of $\ams{\Delta}=\pams{\Delta}$ for reasoning of $\alpha\in L$ from inconsistent $\Delta\subseteq L$ using the generative reasoning model $p(L,M,D;\mu\to1)$. It shows that $\Delta$ has no model satisfying all its formulas. It also shows that every model satisfying all the formulas in a cardinality-maximal consistent subset of $\Delta$ is possible.
\subsection{Empirical reasoning in impossibility}\label{sec:parapossibility}
\begin{figure}[t]
\begin{center}
 \includegraphics[scale=0.06]{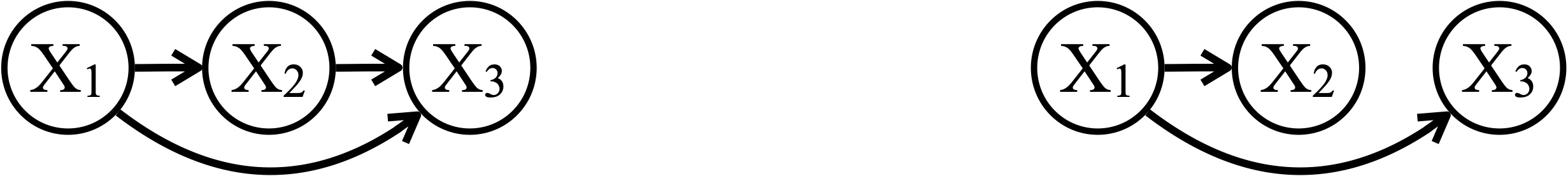}
  \caption{The left can fit with any full joint distribution. The right can fit with any full joint distribution with the conditional independence, $p(X_{3}|X_{2},X_{1})=p(X_{3}|X_{1})$, that rarely holds without data modification.}
  \label{fig:BNs1}
  \end{center}
\begin{center}
 \includegraphics[scale=0.22]{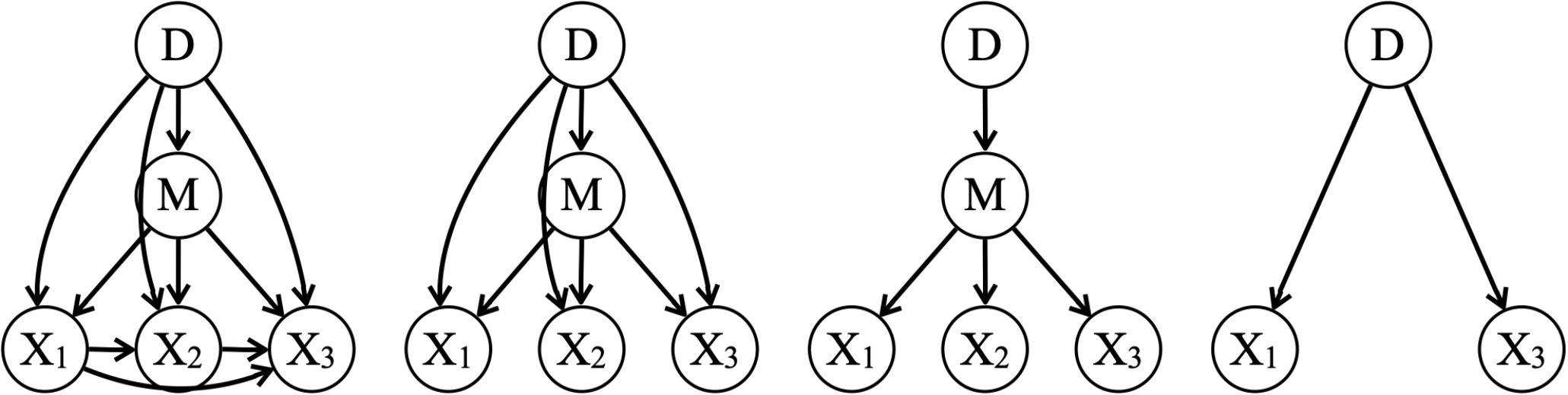}
  \caption{The rightmost structure can be derived from the leftmost one using the properties of formal logic and the natural assumption that each data supports a single model.}
  \label{fig:GLs}
  \end{center}
\end{figure}
\begin{figure*}[t]
\begin{tabular}{cc}
\begin{minipage}{.15\textwidth}
\centering
 \includegraphics[scale=0.28]{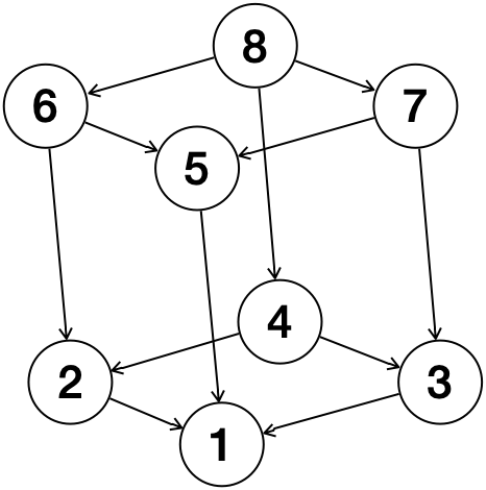}
\end{minipage}
\begin{minipage}{.85\textwidth}
\centering
\scalebox{0.9}{
\begin{tabular}{c|l|l|l}
Node ID & Reasoning type & Logical ground & Grounding assumptions\\\hline
1 & Consistent & $p(\alpha|\Delta)=1$ iff $\Delta\cent\alpha$ & $\ms{\Delta}\neq\emptyset$, $\ms{\Delta}=\pms{\Delta}$\\
2 & Possible & $p(\alpha|\Delta)=1$ iff $\Delta\ent\alpha$ & $\pms{\Delta}\neq\emptyset$\\
3 & Paraconsistent & $p(\alpha|\Delta)=1$ iff $\forall S\in MCS(\Delta). S\cent\alpha$ & $\ams{\Delta}=\pams{\Delta}$\\
4 & Parapossible & $p(\alpha|\Delta)=1$ iff $\forall S\in MPS(\Delta). S\ent\alpha$ & No assumption\\\hline
5,6,7,8 & Uncertain & $p(\alpha|\Delta)\in[0,1]$ generalises the above & Same as ID 1, 2, 3 and 4, resp.
\end{tabular}
}
\end{minipage}
\end{tabular}
\caption{The summary of the logical grounds (right). The generalisation relationship among the eight types of reasoning (left). Consistent reasoning is a special case of all the other types of reasoning, since it has the most strict assumptions. The assumptions and logical grounds of the other types of reasoning can be derived by weakening the assumptions.}
\label{fig:summary}
\end{figure*}
Theorem \ref{thrm:paraconsistency} and Corollary \ref{cor:paraconsistency} depend on the assumption of $\ams{\Delta}=\pams{\Delta}$. In this section, we cancel the assumptions to fully generalise our discussions in Section \ref{sec:paraconsistency}. The generative reasoning model $p(L,M,D;\mu\to1)$ has the following property.
\begin{theorem}\label{thrm:parapossibility}
Let $p(L,M,D;\mu\to1)$ be a generative reasoning model, and $\alpha\in L$ and $\Delta\subseteq L$.
\begin{eqnarray*}
p(\alpha|\Delta)=\frac{\sum_{m\in\pams{\Delta}\cap\pms{\alpha}}p(m)}{\sum_{m\in\pams{\Delta}}p(m)}
\end{eqnarray*}
\end{theorem}
\begin{proof}
Omitted due to space limitation, but same as Theorem \ref{thrm:paraconsistency}. Divide models into $\pams{\Delta}$ rather than $\ams{\Delta}$.
\end{proof}
The following Corollary shows the relationship between the generative reasoning model $p(L,M,D;\mu\to 1)$ and the empirical consequence relation with maximal possible sets.
\begin{corollary}\label{cor:parapossibility}
Let $p(L,M,D;\mu\to1)$ be a generative reasoning model, and $\alpha\in L$ and $\Delta\subseteq L$. $p(\alpha|\Delta)=1$ iff $S\ent\alpha$, for all cardinality-maximal possible subsets $S$ of $\Delta$.
\end{corollary}
\begin{proof}
From Theorem \ref{thrm:parapossibility}, $p(\alpha|\Delta)=1$ iff $\pams{\Delta}\subseteq\pms{\alpha}$. Since $\pams{\Delta}=\bigcup_{S\in MPS(\Delta)}\pms{S}$, $p(\alpha|\Delta)=1$ iff $\bigcup_{S\in MPS(\Delta)}\pms{S}\subseteq\pms{\alpha}$. Therefore, $p(\alpha|\Delta)=1$ iff $\pms{S}\subseteq\pms{\alpha}$, for all $S\in MPS(\Delta)$.
\end{proof}
Note that Theorem \ref{thrm:parapossibility} and Corollary \ref{cor:parapossibility} no longer depend on the assumption of $\ams{\Delta}=\pams{\Delta}$ required in Section \ref{sec:paraconsistency} for Theorem \ref{thrm:paraconsistency} and Corollary \ref{cor:paraconsistency}. Figure \ref{fig:types} illustrates reasoning of $\alpha\in L$ from impossible $\Delta\subseteq L$ using the generative reasoning model $p(L,M,D;\mu\to1)$. It illustrates the most general situation without the assumptions discussed in the previous sections.
\subsection{Probabilistic reasoning}\label{sec:probability}
Let $X_i$ (for $i=1,2,3$) represent three binary random variables corresponding to the propositions `it is raining outside', `the grass is wet', and `the outside temperature is high', respectively. The lower case of each random variable represents its realisation. What one needs in most cases is a posterior probability. For example, $p(x_{1}|x_{3})$ can be represented as follows.
{\small
\begin{align*}
p(x_{1}|x_{3})=\frac{p(x_{1},x_{3})}{p(x_{3})}=\frac{\bm{\sum}_{x_{2}}p(x_{1},x_{2},x_{3})}{\bm{\sum}_{x_{1},x_{2}}p(x_{1},x_{2},x_{3})}
\end{align*}
}
This equation shows that the full joint distribution is required for the exact posterior probability. The space complexity of the full joint distribution is $O(2^N)$ where $N$ is the number of propositions. Thus, the calculation of a posterior probability is generally intractable. \citep{pearl:03} tackled this issue by incorporating the idea of independence. For example, if $X_{3}$ is assumed to be conditionally independent of $X_{2}$ given $X_{1}$, the above equation can be simplified as follows (see Figure \ref{fig:BNs1}).
{\small
\begin{align*}
p(x_{1}|x_{3})&=\frac{\bm{\sum}_{x_{2}}p(x_{3}|x_{2},x_{1})p(x_{2}|x_{1})p(x_{1})}{\bm{\sum}_{x_{1},x_{2}}p(x_{3}|x_{2},x_{1})p(x_{2}|x_{1})p(x_{1})}\\
&=\frac{\bm{\sum}_{x_{2}}p(x_{3}|x_{1})p(x_{2}|x_{1})p(x_{1})}{\bm{\sum}_{x_{1},x_{2}}p(x_{3}|x_{1})p(x_{2}|x_{1})p(x_{1})}
\end{align*}
}
The assumption of independence reduces a space complexity. However, those who strictly adhere to data should not accept the assumption. This is because the assumption rarely holds in reality without a modification of original data or resort to expert knowledge.
\par
Now, let $p(L,M,D;\mu)$ be a generative reasoning model where $L$ is built with $X_{i}$ (for $i=1,2,3$). The posterior probability $p(x_{1}|x_{3})$ can be naively represented as follows (see the leftmost graph in Figure \ref{fig:GLs}).
{\small
\begin{align*}
&p(x_{1}|x_{3})=\frac{p(x_{1},x_{3})}{p(x_{3})}=\frac{\bm{\sum}_{x_{2},m,d}p(x_{1},x_{2},x_{3},m,d)}{\bm{\sum}_{x_{1},x_{2},m,d}p(x_{1},x_{2},x_{3},m,d)}=\\
&\frac{\bm{\sum}_{x_{2},m,d}p(x_{3}|x_{2},x_{1},m,d)p(x_{2}|x_{1},m,d)p(x_{1}|m,d)p(m|d)p(d)}{\bm{\sum}_{x_{1},x_{2},m,d}p(x_{3}|x_{2},x_{1},m,d)p(x_{2}|x_{1},m,d)p(x_{1}|m,d)p(m|d)p(d)}
\end{align*}
}
From Equations (\ref{eq:1}) and (\ref{eq:2}), the above equation can be simplified as follows (see the second and third graphs).
{\small
\begin{align*}
&=\frac{\bm{\sum}_{x_{2},m,d}p(x_{3}|m)p(x_{2}|m)p(x_{1}|m)p(m|d)p(d)}{\bm{\sum}_{x_{1},x_{2},m,d}p(x_{3}|m)p(x_{2}|m)p(x_{1}|m)p(m|d)p(d)}
\end{align*}
}
Each datum has an equal probability, and it supports a single model, i.e., $p(m|d)=1$ if $m=m(d)$ and $p(m|d)=0$ otherwise. Thus, the above equation can be simplified as follows (see the rightmost graph).
{\small
\begin{align*}
&=\frac{\bm{\sum}_{x_{2},d}p(x_{3}|m(d))p(x_{2}|m(d))p(x_{1}|m(d))}{\bm{\sum}_{x_{1},x_{2},d}p(x_{3}|m(d))p(x_{2}|m(d))p(x_{1}|m(d))}\\
&=\frac{\sum_{d}p(x_{3}|m(d))p(x_{1}|m(d))}{\sum_{d}p(x_{3}|m(d))}
\end{align*}
}
\par
Is the outcome reasonable? Let $m_{n}$ denote the $n$th model. Given $p(L,M,D;\mu=1)$, we have 
{\small
\begin{align*}
\textstyle{p(m_{n})=\sum_{k=1}^{K}p(m_{n}|d_{k})p(d_{k})=\frac{1}{K}\sum_{k=1}^{K}p(m_{n}|d_{k})=\frac{K_{n}}{K}},
\end{align*}
}
where $K_{n}$ is the number of data in $m_{n}$. Therefore, $p(m_{n})=K_{n}/K$ is equivalent to the maximum likelihood estimate (MLE). Moreover, we have
{\small
\begin{align*}
p(\alpha|\Delta)=\frac{p(\alpha,\Delta)}{p(\Delta)}=\frac{\sum_{n=1: m_{n}\in\ms{\alpha}\cap\ms{\Delta}}^{N}p(m_{n})}{\sum_{n=1: m_{n}\in\ms{\Delta}}^{N}p(m_{n})}.
\end{align*}
}
The denominator (resp. numerator) is thus the sum of the MLE of the models where $\Delta$ (resp. $\alpha$ and $\Delta$) is true. The space complexity is now $O(K)$ where $K$ is the number of data. It is not reasoning using the full joint distribution over models. The equation ignores all models without data support. The exponential summation over models thus can be reduced to the liner summation over data.
\section{Conclusions}
Symbolic knowledge is an abstraction of data. This simple idea caused a probabilistic model of how data cause symbolic knowledge in terms of its satisfiability in formal logic. Figure \ref{fig:summary} summarises the logical grounds and assumptions of all the types of logical reasoning studied in this paper.
\bibliography{btx_kido}
\end{document}